\documentclass[10pt, a4paper, onecolumn, twoside]{amsart}

\usepackage{titlesec, color, amsthm}
\definecolor{blue1}{rgb}{0,0,0}
\renewcommand\thesection{\arabic{section}}
\renewcommand{\theequation}{\arabic{equation}}

\titleformat{\section}[hang]{\color{blue1}\large\bfseries\sffamily}{\thesection}{0mm}{. }[]
\titleformat{\subsection}[hang] {\color{blue1}\bfseries\sffamily}{\thesubsection}{0em}{. }[]
\titleformat{\subsubsection}[hang] {\color{blue1}\sffamily}{\thesubsubsection}{0em}{. }[]
\titlespacing*{\section}{1em}{3.5ex plus .2ex minus .2ex}{1ex plus .2ex}
\titlespacing*{\subsection}{0em}{3ex plus .2ex minus .2ex}{1ex plus .2ex}
\titlespacing*{\subsubsection}{0em}{3ex plus .2ex minus .2ex}{1ex plus .2ex}

\renewenvironment{abstract}{{\color{blue1}\small\bfseries Abstract.}\footnotesize}{\par \vskip .1in}
\makeatletter
\def\@setauthors{
\begingroup 
\def \thanks{\protect\thanks@warning}
\trivlist \centering\footnotesize \@topsep30\p@\relax \advance\@topsep by -\baselineskip
\item\relax \author@andify \authors \def\\{\protect\linebreak} {\color{blue1}\large\authors} \endtrivlist \endgroup}
\def\@settitle{\centering{\color{blue1} \Large \bfseries \bfseries \@title \par}}
\makeatother
\usepackage[normal, small, up, textfont=it]{caption}
\usepackage[utf8]{inputenc}
\usepackage{dsfont}
\usepackage{hyperref}
\usepackage{enumerate}
\usepackage{bbold}
\usepackage{graphicx}
\graphicspath{{figures}}
\usepackage{amsmath,amsthm,amssymb}
\usepackage{mathtools,mathdots,mathrsfs}
\usepackage{cleveref}
\usepackage{verbatim}
\usepackage{bm}
\usepackage{subcaption}
\usepackage{cite}
\usepackage[dvipsnames,table]{xcolor}
\usepackage{pgfplots}
\usepackage[ruled]{algorithm2e}
\usepackage{algpseudocode}
\usepackage{rotating, amsfonts}
\usepackage{verbatim}
\usepackage{epstopdf}
\usepackage{multirow}

\newtheorem{theorem}{Theorem}[section]

\renewcommand{\leq}{\ensuremath{\leqslant}}
\renewcommand{\geq}{\ensuremath{\geqslant}}

\renewcommand{\vec}[1]{\ensuremath{\bm{#1}}}

\newtheorem{definition}[theorem]{Definition}

\newtheorem{proposition}[theorem]{Proposition}
\newtheorem{corollary}[theorem]{Corollary}

\algnewcommand\algorithmicinput{\textbf{Input:}}
\algnewcommand\Input{\item[\algorithmicinput]}

\algnewcommand\algorithmicoutput{\textbf{Output:}}
\algnewcommand\Output{\item[\algorithmicoutput]}

\newcommand{\norm}[1]{\ensuremath{\left\| #1\right\|}}

\newcommand{\vect}[1]{\boldsymbol{#1}}
\newcommand{\matr}[1]{\boldsymbol{#1}}
\newcommand{\eqdef}{\stackrel{\textrm{def}}{=}}

\newcommand{\E}{\mathrm{E}} 

\newcommand{\bu}{\vect{u}}

\newcommand{\bM}{\matr{M}}

\newcommand{\bA}{\matr{A}}       
\newcommand{\bQ}{\matr{Q}}       
\newcommand{\bP}{\matr{P}}

\newcommand{\bB}{\matr{B}}

\newcommand{\bU}{\matr{U}}

\newcommand{\bI}{\matr{I}}       
\newcommand{\bV}{\matr{V}}       
       
\newcommand{\bK}{\matr{K}}

\newcommand{\bS}{\matr{S}}       
\newcommand{\bL}{\matr{L}}
\newcommand{\R}{\mathbb{R}}

\newcommand{\X}{\mathcal{X}}

\newcommand{\Y}{\mathcal{Y}}

\renewcommand{\O}{\mathcal{O}}       
\newcommand{\Otilde}{\tilde{\O}}

\DeclareMathOperator{\mspan}{span}

\definecolor{darkgreen}{rgb}{0,0.6,0}

\newcommand{\bx}{\vect{x}}
\newcommand{\bs}{\vect{s}}
\newcommand{\bz}{\vect{z}}

\newcommand{\pione}{\pi^{(1)}}
\newcommand{\pit}{\pi^{(t)}}

\newcommand{\longtitle}{{A Faster Sampler for\\ Discrete Determinantal Point Processes}}

\newcommand{\NTlong}{{Nicolas Tremblay}}
\newcommand{\SBlong}{{Simon Barthelm\'e}}
\newcommand{\POAlong}{{Pierre-Olivier Amblard}}
\newcommand{\NTshort}{{N.~Tremblay}}
\newcommand{\SBshort}{{S.~Barthelm\'e}}
\newcommand{\POAshort}{{P-O.~Amblard}}
\newcommand{\CNRS}{{CNRS, Univ Grenoble-Alpes, Gipsa-lab, France}}

\title[\longtitle]{\longtitle}
\author[\SBshort]{\SBlong}
\author[\NTshort]{\NTlong}
\author[\POAshort]{\POAlong}
\thanks{All three authors are with \CNRS.}

\begin{document}

\maketitle
\begin{abstract}
   Discrete Determinantal Point Processes (DPPs) have a wide array of potential
  applications for subsampling datasets. They are however held back in some cases by the
  high cost of sampling. In the worst-case scenario, the sampling cost scales as $\O(n^3)$ where $n$ is the number of elements of the ground set. A popular workaround to this prohibitive cost is to sample DPPs defined by low-rank kernels. In such cases, the cost of standard sampling algorithms scales as $\O(np^2 + nm^2)$ where $m$ is the (average) number of samples of the DPP (usually $m \ll n$) and $p$  the rank of the kernel used to define the DPP ($m\leq p\leq n$). The first term, $\O(np^2)$, comes from a SVD-like step. 
  We focus here on the second term of this cost, $\O(nm^2)$, and show that it 
  can be brought down to $\O(nm + m^3 \log m)$ without loss on the sampling's exactness. 
  In practice, we observe very substantial 
  speedups compared to the classical algorithm as soon as $n>1,000$.
  The algorithm described here is a close variant of the standard algorithm
  for sampling continuous DPPs, and uses rejection sampling. 
  In the specific case of projection DPPs, we also show that any additional sample can be drawn in time $\O(m^3 \log m)$. \\
  Finally, an interesting
  by-product of the analysis is that a realisation from a DPP is typically
  contained in a subset of size $\O(m \log m)$ formed using leverage score i.i.d. sampling. 
\end{abstract}

Discrete Determinantal Point Processes have been advocated as a way of
subsampling large datasets, because they produce samples that preserve some of
the diversity of the original dataset~\cite{kulesza2012determinantal}. One impediment to their broad adoption
in practice lies in their computational cost; aside from some special cases
(\textit{e.g.}, random spanning forests~\cite{avena_two_2017}), exact sampling of a DPP with a large number
of elements is rather expensive.

In this manuscript, we show that a simple modification of the standard algorithm
yields a substantial improvement, without loss on the algorithm's exactness. The
modification we suggest consists in using a form of rejection sampling. The idea is
not very original, and indeed appears in works by~\cite{lavancier2015determinantal} for
continuous DPPs, and more recently by~\cite{derezinski2019minimax} in the context of experimental design. What is striking is how effective this modification can be, when
sampling discrete DPPs, especially given how easy it is to implement. 

Here and throughout, let $n$ designate the
size of the ground set the DPP draws from, and $m$ be the (average) size of the
subsample produced by the DPP (usually $m\ll n$). In the worst-case, the cost of producing a sample may be as high as $\O(n^3)$, as it requires a full diagonalisation of the
kernel~\cite{Hough:DPPandIndep}. However, in the more realistic context
of low-rank kernels and using standard exact sampling algorithms, this figure
drops to $\O(np^2 + nm^2)$ where $p$  is the rank of the kernel ($m\leq p\leq n$)~\cite{gillenwater2014approximate}. We improve this to $\O(np^2 + nm + m^3 \log
m)$. We readily see that, even though this is an improvement for sampling any
DPP, the closer is $p$ to $m$, the more substantial the improvement in practice,
as $np^2$ stays the headline complexity. We identify three popular and general
use-cases for which rejection sampling brings substantial speed-ups compared to the classical algorithm:
\begin{enumerate}
	\item \emph{(very significant speed-up: $p=m$ and orthogonalisation is
      already computed)} sample a DPP with kernel $\bK=\bQ\bQ^\top$, where
    $\bQ\in\mathbb{R}^{n\times m}$ is orthonormal ($\bQ^\top\bQ=\bI$) and given
    (for instance, the DPPs used by~\cite{launay2021determinantal}). As there is
    no orthogonalisation to compute, the total cost using rejection
    sampling is $\O(nm + m^3 \log m)$, substantially faster than the usual cost in $\O(nm^2)$.
	\item \emph{(significant speed-up: $p=m$ and orthogonalisation has yet to be computed)} in some cases, the orthonormal basis $\bQ$ is not known from the start. A popular context is when one wishes to sample a fixed-size L-ensemble of size $m$ with $\bL=\bV\bV^\top$ with  $\bV\in\mathbb{R}^{n\times m}$  a matrix of features. In this case, one i/~first computes an orthonormal basis $\bQ$ of the span of $\bV$, before ii/~sampling a DPP with kernel $\bK=\bQ\bQ^\top$ (as in the previous case). Step i/ involves, e.g., a QR decomposition. Even though the cost of QR scales theoretically as $\O(nm^2)$, 
	it is highly efficient (and parallelisable) in modern hardware such that the bottleneck in previous state-of-the-art is step ii/. Our improvement of step ii/ thus also has  practical (possibly very large) speed-ups in this context. 
	In addition, there are special cases of feature matrices $\bV$ for which
  computing an orthogonal basis has cost less than $\O(nm^2)$; increasing
  further the potential benefits of our approach. This is the case, \textit{e.g.},
  for some classes of sparse  $\bV$ \cite{davis2008algorithm}. 
	\item \emph{(moderate speed-up: $p$ equals a few times $m$)} Same context as 2/ but in the case where the feature matrix $\bV$ is of size $n\times p$ with $p>m$. In this case, step i/~is to compute the SVD of $\bV$. If $p$ is too large, this will be  the dominant step and our improvement over step ii/~will be negligible. However, in popular cases where $p$  is only a few times $m$, the speed-up is appreciable (see Section~\ref{sec:empirical-results} for details).
\end{enumerate}

Moreover, in contexts where one needs several realisations of the same DPP, step i/ is computed once, and step ii/ as many times as the number of samples needed; such that our improvement over step ii/ becomes that much more useful.

\emph{Organisation of the paper.} Section \ref{sec:background} briefly
introduces the main objects and the state-of-the-art. Section
\ref{sec:accept-reject} describes our algorithm and its runtime, and  Section~\ref{sec:empirical-results} presents empirical results. A corollary of our result states that DPPs are typically contained in a i.i.d. sample of size
$\O(m \log m)$. Section~\ref{sec:m-vs-mlogm} discusses this fact and offers concluding remarks.

\section{Background and notation}
\label{sec:background}

\subsection{Discrete DPPs}
\label{sec:discrete-dpps}

For more background on discrete DPPs, we refer readers to~\cite{kulesza2012determinantal} and \cite{tremblay2022extended}. Discrete DPPs
are a specific instance of a \emph{discrete point process}. We say $\X$ is a
discrete point process on a ground set $\Omega$, if it is a random subset of
$\Omega$. Without loss of generality, we let $\Omega = \{1, \ldots,n \}$ so that
$\X$ is a random subset of indices. Also, for two matrices $\bA$ and $\bB$ of same size, the notation $\bA\preceq\bB$ means that $\bB-\bA$ is positive semidefinite.
\begin{definition}[DPP]
	$\X$ is a DPP on $\Omega$
	with marginal kernel $\bK \in \R^{n \times  n}$ such that $\mathbf{0}\preceq \bK \preceq \bI$, noted $\X \sim DPP(\bK)$, if for all fixed subsets
	$S \subseteq \Omega$, we have
	\begin{equation}
		\label{eq:incl-prob}
		p(S \subseteq \X) = \det \bK_{S}
	\end{equation}
\end{definition}

Here $\bK_{S}$ is the principal submatrix of $\bK$ with indices given by $S$.
We shall use ``Matlab'' notation, where $\bK_{A,B}$ denotes the submatrix with
row indices $A$ and column indices $B$, $\bK_{A,:}$ means all columns and
$\bK_{:,B}$ all rows.

\begin{definition}[Projection DPP]
	A projection DPP is a DPP whose kernel is a projection matrix (ie.
	$\bK^2 = \bK$).
\end{definition}

If $\bK$ is a projection matrix, it can be written as $\bQ\bQ^\top$ where $\bQ \in
\R^{n \times  m}$ is an orthonormal basis for $\mspan \bK$  ($\bQ^\top\bQ = \bI$
and $\mspan \bQ = \mspan \bK$). Note that any orthonormal basis for $\bK$ is
enough, $\bQ$  need not be a basis of eigenvectors. 

Projection DPPs are important because of the following mixture decomposition,
due to \cite{Hough:DPPandIndep}. 
\begin{theorem}[mixture representation]\label{houghs_thm}
	Let $\X \sim DPP(\bK)$, and $\bK = \bU \matr{\Lambda} \bU^\top$ the
	eigendecomposition of $\bK$, with $\matr{\Lambda}$ the diagonal matrix of eigenvalues $\{\lambda_j\}_{j=1,\ldots,n}$ and $\bU=\left(\bu_1|\bu_2|\ldots|\bu_n\right)$ the matrix of eigenvectors. Then the following process produces a sample from
	$\X$:
	\begin{enumerate}
		\item Sample a subset $\Y$ of eigenvectors by including each eigenvector
		$\bu_j$ with probability $\lambda_j$
		\item Form the projection kernel $\bK_{\Y} = \bU_{:,\Y}(\bU_{:,\Y})^\top$
		\item Sample $\X \sim DPP(\bK_{\Y}) $
	\end{enumerate}
\end{theorem}

The cost of sampling a DPP when following this recipe equals the cost of
computing the eigendecomposition of $\bK$ ($\O(np^2)$ with $p$ the rank of
$\bK$) followed by the cost of sampling a projection DPP (step 3). It is the latter step that we focus on here. 

\subsection{Fixed-size DPPs}
The cardinal of a DPP is in general random. Such varying-sized samples are
not practical in many applications, 
which led ~\cite{kulesza2011k} to define fixed-size DPPs\footnote{They are often called k-DPPs in the literature, but we prefer ``fixed-size DPPs'' in order not to overload the symbol $k$ too much.}
\begin{definition}[Fixed-size DPP]
	\label{def:fsDPP}
	A fixed size DPP of size $m$ is a DPP $\X$ conditioned on $|\X|=m$. 
\end{definition}

To sample a fixed-size DPP with kernel $\bK$, one follows the same recipe as in Theorem~\ref{houghs_thm} except for the first step that is replaced by:
\begin{enumerate}
	\item \emph{Sample a subset $\Y$ of eigenvectors by including each eigenvector
	$\bu_j$ with probability $\lambda_j$; \emph{conditioned on $|\Y|=m$.}}
\end{enumerate}
Performing such a conditioned sampling can be done by Algorithm 8
of~\cite{kulesza2012determinantal}, which works by explicitly computing elementary polynomials. If $n$ and/or $m$ are too large, numerical instabilities usually arise with this method, and ~\cite{barthelme_asymptotic_2019} propose a way to stabilize this conditioned sampling.

\subsection{L-ensembles and fixed-size L-ensembles}
\label{sec:ell-ensembles}

L-ensembles are a subclass of DPPs popular in machine learning applications
because of their intuitive definition:
\begin{definition}[L-ensemble]
	Let $\bL$ be a positive semi-definite matrix. 
	$\X$ is a L-ensemble on $\Omega$ if for all $X \subseteq \Omega$
	\begin{equation}
		\label{eq:ell-ensemble}
		p(\X = X) = \frac{1}{\det (\bI + \bL)} \det(\bL_X)
	\end{equation}
\end{definition}

L-ensembles are specified via their likelihood function, Eq.~\eqref{eq:ell-ensemble},
which states that those subsets of $\Omega$ where the submatrix $\bL_X$ is
well-conditioned, are preferred. Intuitively, if $L_{i,j}$ represents a
similarity between items $i$ and $j$ of $\Omega$, then the L-ensemble favours
subsets of $\Omega$ that hold dissimilar items. 

As with DPPs, one defines fixed-size L-ensembles as:
\begin{definition}[Fixed-size L-ensemble]
	\label{def:fsLens}
	A fixed size L-ensemble of size $m$ is a L-ensemble $\X$ conditioned on $|\X|=m$. 
\end{definition}

(Fixed-size) L-ensembles are (fixed-size) DPPs with kernel $\bK =(\bI+\bL)^{-1} \bL$~\cite{kulesza2012determinantal, tremblay2022extended}, and the mixture
representation thus applies. 

%

To conclude this section, we have seen that all (fixed-size) L-ensembles and
more generally all (fixed-size) DPPs have a mixture representation that divides
the sampling algorithm in two steps: i/ a diagonalisation step that costs
$\mathcal{O}(np^2)$, ii/ a step consisting of sampling a projection DPP. Step
ii/ is known\footnote{This is an average (resp. deterministic) cost for DPPs
  (resp. fixed-size DPPs) for which $m$ refers to the average (resp. desired)
  size of the sample.} to cost $\mathcal{O}(nm^2)$ in the literature. The
purpose of this paper is to show that the cost of this second step can always
(and easily) be reduced to $\O(nm + m^3 \log m)$. 

\subsection{State-of-the-art}
\label{sec:SOTA}

Various directions have been explored when designing fast samplers for discrete
DPPs. Some have focused on bypassing the eigendecomposition of $\bK$ or $\bL$
\cite{poulson2020high,launay2020exact,derezinski2019exact}.
If $\bK$ is a sparse matrix, then the algorithms in \cite{poulson2020high} can
be quite advantageous compared to standard $\O(n^3)$ algorithms. These
algorithms are difficult to adapt to L-ensembles in an efficient manner (for
instance, they cannot take advantage of sparsity in $\bL$). Random spanning
forests \cite{avena_two_2017} are an example of a discrete DPP where a fast
sampler (Wilson's algorithm,~\cite{wilson_generating_1996}) is available, and
sparsity in $\bL$ seems to play a role. However, Wilson's algorithm does not
extend readily to L-ensembles with arbitrary structures.

Another direction for generic DPP samplers is to give up on exactness.
Approximate samplers are available, based on Markov Chain Monte Carlo methods.
The most recent results in that direction are in \cite{anari2022optimal}, where
the authors show that given some preprocessing there are MCMC samplers that run
in $\Otilde(m^\omega)$, where the $\Otilde$ is shorthand for $\O(.)$ ``up to
logarithmic factors'', and $\omega$ is the exponent of matrix-multiplication
time, which for practical values of $m$ is effectively 3. The pre-processing
consists essentially in estimating the inclusion probabilities, also known as
the leverage scores, and its runtime is given by \cite{anari2022optimal} as
$\Otilde(nm^{\omega - 1})$. Our results are essentially the same (preprocessing
in $\O(nm^2)$, sampling in $\Otilde(m^3)$), but our sampler is exact. 

Also, two papers \cite{han2022mcmc,han2022scalable} extend the
tree-based approach of~\cite{gillenwater2019tree} to obtain both 
approximate and exact samplers with complexity slightly larger than our
proposal. They are also more complicated to implement. However, they can
handle non-symmetric DPPs, which the algorithm given here cannot do.

Finally, the use of rejection sampling is not new in the context of DPP sampling, since algorithms for sampling continuous DPPs use this strategy out of necessity
\cite{lavancier2015determinantal}. More recently, \cite{derezinski2019minimax}
describe a similar algorithm to ours, in the context of volume sampling for experimental design. Our contribution compared to~\cite{derezinski2019minimax} is to i/~lay out a more refined analysis: better bound on the total number of proposals in Theorem~\ref{thm:runtime}, novel investigation in the theoretical implications of this theorem in Section~\ref{sec:m-vs-mlogm}, ii/~ keep an eye on practical implementations: see  Sections~\ref{sec:refinements} and~\ref{sec:empirical-results}. What we would like to stress is that rejection sampling leads to an algorithm that is much faster in practice, but no more complicated to implement, than the traditional discrete sampler.

\section{Sampling via accept-reject}
\label{sec:accept-reject}

In this section the goal is to formulate and analyse an algorithm for sampling a
projection DPP $\X \sim DPP(\bQ\bQ^\top)$ with $\bQ \in \R^{n \times m}$, verifying $\bQ^\top\bQ = \bI$. The
first exact such algorithm was described by \cite{Hough:DPPandIndep}, and
adapted in \cite{kulesza2012determinantal} to the discrete case. The first
efficient version appeared in \cite{gillenwater2014approximate}. 
It is effectively a variant of the Gram-Schmidt algorithm.

For completeness we give an easy derivation of this classical algorithm in the next section (Section~\ref{subsec:SOTA_alg}), and the notation will serve to
describe our own variant, in Section~\ref{subsec:rejection-sampling}.

\subsection{State-of-the-art algorithm}
\label{subsec:SOTA_alg}
A projection DPP has size $m$ almost surely~\cite{kulesza2012determinantal}. We shall sample the $m$
elements successively. Let $\vec{\X} = (x_1,\ldots,x_m)$ be an ordered version of
$\X$; we can go from $\vec{\X}$ to $\X$ by forgetting the order and from $\X$ to
$\vec{\X}$ by ordering randomly. The latter can be achieved by picking a first
item uniformly from $\X$, then a second, then a third etc.
The sampling algorithm proceeds via the following decomposition:
\begin{equation}
	\label{eq:chain-rule}
	p(\vec{\X}) = p(x_1)p(x_2|x_1)p(x_3|x_1,x_2)\dots p(x_m|x_1 \dots x_{m-1})
\end{equation}
The algorithm samples $x_1$ first, then $x_2$ given $x_1$ has been selected, etc. 
The law of $x_1$ is the law of the first element of $\vec{\X}$, a randomly
ordered version of $\X$. That is equivalent to $x_1$ being sampled uniformly at
random from $\X$, and so:
\[ p(x_1 = i) = \frac{1}{m}p(i \in \X) = \frac{K_{i,i}}{m} \]

We can similarly obtain the law of $x_2$ given $x_1$, as two elements drawn
randomly (without replacement) from $\X$:
\begin{align*}
p(x_2 = i|x_1=j) &= \frac{p(x_1 = j, x_2=i)}{p(x_1=j)}\\
&=
\frac{p(i \in \X , j \in \X) }{m(m-1)} \times \frac{m}{p(j \in \X) } \\
&=
\frac{1}{(m-1)} \frac{\det \bK_{\{i,j\}}}{K_{j,j}} 
\end{align*}  

The formula for determinants of block matrices yields:
\[ \frac{\det \bK_{\{i,j\}}}{K_{j,j}} = K_{i,i}- \frac{K_{i,j}^2}{K_{j,j}} \]
and so:
\[ p(x_2 = i|x_1=j) = \frac{1}{(m-1)} \left(K_{i,i}- \frac{K_{i,j}^2}{K_{j,j}} \right)\]

For the general term in the chain rule decomposition (Eq.
(\ref{eq:chain-rule})), the same reasoning applies. We obtain:
\begin{align}
	\label{eq:chain-rule-gen-term}
	p \left(x_t = i| \vec{\X}_{1:(t-1)} = X_t \right) &\\= \frac{1}{(m-t)} &\left( K_{i,i} - \bK_{i,X_t}(\bK_{X_t})^{-1}\bK_{X_t,i} \right) \nonumber
\end{align}

Eq.~(\ref{eq:chain-rule-gen-term}) is enough to give us a sampling algorithm,
since at each step we have an explicit (discrete) probability distribution to
sample from. However, implementing Eq.~(\ref{eq:chain-rule-gen-term}) naïvely, we would be computing a matrix
inverse at each step, which would turn out to be quite expensive for large $m$. To get an efficient algorithm a bit more work is needed.

Let us reexpress Eq.~(\ref{eq:chain-rule-gen-term}) in terms of $\bQ$. Recalling
$\bK = \bQ\bQ^\top$, we obtain:
\begin{align}
	\label{eq:chain-rule-Q}
	p \left(x_t = i| \vec{\X}_{1:(t-1)} = X_t \right)& \\= \frac{1}{ (m-t)}
	&\left( K_{ii} - \bQ_{i,:}\bM_t(\bQ_{i,:})^\top  \right)\nonumber
\end{align}
where $\bM_t = (\bQ_{X_t,:})^\top \left(
\bQ_{X_t,:}(\bQ_{X_t,:})^\top \right)^{-1} \bQ_{X_t,:}$ 
is a projection matrix ($\bM_t^2 = \bM_t$) of size $m$ and rank $|X_t| = t-1$, and so can be
rewritten
\[ \bM_t = \sum_{i=1}^{t-1} \bs_{i}\bs_{i}^\top \]
where $\bs_1 \dots \bs_t$ form an orthonormal basis for $\mspan \bM_t = \mspan
\bQ_{X_t,:}^\top$, the linear subspace spanned by the rows of $\bQ$ selected so
far. 
A first source of computational savings comes from realising that $\bM_t$ can be
computed iteratively via the Gram-Schmidt process.
Notice that $\bM_t = \bM_{t-1} + \bs_{t-1}\bs_{t-1}^\top$, and $\bM_{t-1}$ spans
$\mspan \bQ_{X_{t-1},:}^\top$. We obtain $\bs_{t-1}$ via Gram-Schmidt: first we
compute the residual 
\[ \bz_{t-1} = (\bI - \bM_{t-1}) \bQ_{x_{t-1},:}^\top\]
and then we normalise:
\[ \bs_{t-1} = \frac{\bz_{t-1}}{\norm{ \bz_{t-1}} }\]
At each step $t$ this costs $\O(m(t-1))$ operations, and we will need to do this
$m-1$ times at a total cost of $\O(m^3)$. 

Next, we can show that the probability distribution we sample from at step $t$
can be easily obtained from the one we had at step $t-1$. It is more convenient
to write this using unnormalised versions of the densities.
Let $\pi^{(1)}(i) = m p(x_1 = i) = K_{i,i}$. Next, we define:
\begin{align*}
 \pi^{(2)}(i) &= (m-1)p(x_2 = i | x_1 = j) \\
 &= K_{i,i} -
\frac{K_{i,j}^2}{K_{i,i}} = \pi^{(1)}(i) - \frac{K_{i,j}^2}{K_{i,i}} \end{align*}
Note that we have suppressed the dependency on the past in the notation
$\pi^{(2)}(i)$: it is to be understood as the (unnormalised) density we draw
from at the second step of the algorithm. 
In the general case, we define:
\begin{equation}
	\label{eq:unnormalised-general}
	\pi^{(t)}(i) = (m-t+1) \;p\left(x_t = i | \vec{\X}_{1:(t-1)}=X_t \right)
\end{equation}
Injecting Eq.~(\ref{eq:chain-rule-gen-term}) and Eq.~(\ref{eq:chain-rule-Q}), we find
\begin{align}
	\label{eq:unnormalised-recursion}
	\pi^{(t)}(i) &= \pi^{(1)}(i) - \bQ_{i,:} \bM_t (\bQ_{i,:})^\top \nonumber\\
	&= \pi^{(1)}(i) - \bQ_{i,:} (\bM_{t-1} + \bs_{t-1}\bs_{t-1}^\top) (\bQ_{i,:})^\top  \nonumber\\
	&= \pi^{(t-1)}(i) - (\bQ_{i,:}\bs_{t-1})^2
\end{align}
All we need to do at each step of the algorithm is to
\begin{enumerate}
	\item pick an item $i$ according  to $\pi^{(t)}$ 
	\item perform a step of the Gram-Schmidt algorithm to update $\bM_t$ based on
	the new vector $\bQ_{i,:}$
	\item Update the probability distribution to $\pi^{(t+1)}$ according to Eq.~(\ref{eq:unnormalised-recursion})
\end{enumerate}

Sampling from a discrete distribution (step 1 above) can be done at cost
$\O(n)$, and is needed $m$ times, for a total cost of $\O(nm)$. We have already
established that the cost of the Gram-Schmidt algorithm is $\O(m^3)$. It is the
update to the probability distribution that is the most costly, with each step
costing $\O(nm)$ ($n$ dot products in $\R^m$) for a total of $\O(nm^2)$. Since
$n>m$ the cost of the algorithm scales as $\O(nm^2)$. We show pseudo-code for this standard
algorithm as Alg.~\ref{alg:sampling-proj}. 

\begin{algorithm}[t]
	\caption{Sampling from a projection DPP $\X \sim DPP(\bK=\bQ\bQ^\top)$, standard algorithm}
	Initialise $\pi(x) \gets \sum_{j=1}^m Q_{x,j}^2$, $\X = \emptyset$, Gram-Schmidt
	basis $\bS=[]$ \;
	\ForEach{$t \in 1\ldots m$}
	{
		Sample $x$ from $ \frac{\pi(x)}{m-(t-1)}$, add to $\X$ \;
		Compute residual $\bz_t = (\bI - \bS\bS^\top)(\bQ_{x,:})^\top$ \;
		Add column $\bs_t = \frac{\bz_t}{\norm{\bz_t}}$ to $\bS$ \;
		Compute $\vect{v} = \bQ\bs_{t} \in \R^n$ \;
		Update probabilities $\pi(x) \gets \pi(x) - (v_{x})^2$\;
	}
	\label{alg:sampling-proj}
\end{algorithm}

In the next section, we move on to the core of our contribution, showing that this $\O(nm^2)$ cost can be reduced to $\O(nm + m^3 \log m)$ via rejection sampling.

\subsection{Using rejection sampling}
\label{subsec:rejection-sampling}

As discussed above, the most expensive part of alg.
\ref{alg:sampling-proj}, lies in updating the probability distribution to sample
from (last step of the \emph{for} loop). It turns out that
the accept-reject method lets us bypass this step.

To briefly recall the rejection sampling idea, suppose we have an unnormalised
density $\pi(x)$ (the target) we wish to draw from, and a proposal $q(x)$, also
unnormalised, but that we know how to draw from, and is not too far from $\pi$.
Further, $q(x)$ has support at least as wide as $\pi(x)$, and upper bounds it
($\pi(x) \leq q(x)$ over the support). We may then draw $x$ from $q(x)$ and
accept the sample with probability $\frac{\pi(x)}{q(x)}$. The accepted samples then have
density $\pi(x)$. If $q(x)$ is a good bound for $\pi(x)$, then the rejection
sampler will be quite efficient. In the limit where $\pi = q$, the acceptance
probability goes to 1. If on the other hand $q(x)$ is quite loose, the
acceptance probability may be bad. 

In the DPP sampling algorithm, we need to sample from $\pi^{(1)}$, then
$\pi^{(2)}$, etc. up to $\pi^{(m)}$. Our proposal is to compute $\pi^{(1)}$
exactly for all $n$ entries, then use $\pi^{(1)}$ as proposal distribution for
the rest of the sequence. 

Recall that  $\pi^{(t)}$ is the unnormalised density defined in Eq.
(\ref{eq:unnormalised-general}). From the recursion in Eq.
(\ref{eq:unnormalised-recursion}), we have that
\[ \pi^{(t)}(x) \geq \pi^{(t+1)}(x) \]
for all $x$ and $1\leq t \leq m-1$. This is true in particular for $\pi^{(1)}$, which
can therefore be used as a proposal distribution for all subsequent $\pi^{(t)}$. 
We can directly compute the probability of accepting a proposed sample. At step $t$, we
sample $x$ from the normalised density $\frac{1}{m}\pi^{(1)}(x)$ and accept it
with probability $\frac{\pi^{(t)}(x)}{\pi^{(1)}(x)}$. The acceptance probability
equals:

\begin{equation}
	\label{eq:acceptance-prob}
	\rho_t =  \sum_{i=1}^n \frac{\pi^{(t)}(i)}{\pi^{(1)}(i)}\frac{\pi^{(1)}(i)}{m} = \frac{1}{m} \sum_{i=1}^n \pi^{(t)}(i) = \frac{(m-t+1)}{m}
\end{equation}

This probability decreases at each step of the algorithm, but at
the final step it is still positive and equals $\frac{1}{m}$.

Let us outline the proposed algorithm. First, one computes every entry of
$\pione$. For this we use the following formula
\begin{equation}
	\label{eq:leverage}
	\pione(i) = K_{i,i} = \sum_{i=1}^n Q_{ij}^2
\end{equation}
The computation is equivalent to computing the norm of each row of $\bQ$, at
cost $\O(nm)$. Because we need to sample from $\pione$ repeatedly, it pays to
use Walker's alias method \cite{walker1977efficient} (see also Chapter III.4 of~\cite{devroye86nonuniformrandom}). Given a preprocessing cost
of $\O(n)$, the alias method gives us all subsequent samples at cost
$\O(1)$ instead of $\O(n)$.

At the first step we sample our first item from $\pione$. At step 2, and all
subsequent steps, we use rejection sampling, which involves computing the ratio
\begin{equation}
	\label{eq:acc-ratio}
	\frac{\pit(x)}{\pione(x)}= 1 - \frac{1}{\pione(x)}\sum_{j=1}^{t-1} (\bQ_{x,:} \bs_j )^2
\end{equation}
by Eq.~(\ref{eq:unnormalised-recursion}). Computing the acceptance ratio has
cost $\O(m(t-1))$ at step $t$, the cost of $t-1$ dot products in $\R^m$. We do
this repeatedly until a proposal is accepted, at which point we need to perform
a Gram-Schmidt step to update $\bM_t$ to $\bM_{t+1}$. We then move on to the
next iteration, or stop if $t=m$. 

We summarise the whole process as Alg.~\ref{alg:sampling-proj-rejection}.  
To recapitulate the different computational costs:
\begin{itemize}
	\item Preprocessing cost: computing $\pione$ for all
	entries comes at cost $\O(nm)$ and setting up Walker's alias method at cost $\O(n)$
	\item The Gram-Schmidt process (computing $\bz_t$ then $\bs_t$) costs $\O(tm)$ at step $t$. Summing this figure for $t=1$ to $m$ gives a cost of $\O(m^3)$
	\item We now need to compute the average cost of the \emph{while} loop. 
	At step $t$, the rejection sampler has probability
	$\rho_t$ of succeeding, given by Eq.~(\ref{eq:acceptance-prob}). The number of
	proposals $R_t$ that are required until acceptance is thus a random variable that follows a geometric
	distribution with success probability $\rho_t$. One thus has:
	\begin{equation}
		\label{eq:expected-number-of-props}
		\E(R_t) = \frac{1}{\rho_t} = \frac{m}{m-t+1}.
	\end{equation}
	Since computing the acceptance ratio costs $\O(mt)$ for each trial, the expected cost of the \emph{while} loop at
	step $t$ scales as $\O\left(\frac{m^2t}{m-t+1}\right)$. Summing this
  figure over $t$: $$\sum_{t=1}^m\frac{m^2 t}{m-t+1}\leq
  m^3\sum_{t=1}^m\frac{1}{m-t+1}$$ yields a total expected cost scaling as \footnote{$\sum_{t=1}^m\frac{1}{m-t+1}=\sum_{t=1}^m\frac{1}{t}$ scales as $\O(\log m)$: see, \emph{e.g.}, Chapter 6 of~\cite{abramowitz1964handbook}} $\O(m^3\log m)$.
\end{itemize}

\begin{algorithm}[t]
	\caption{Sampling from a projection DPP $\X \sim DPP(\bK=\bQ\bQ^\top)$ with
		rejection sampling }
	Initialise $\pione(x) \gets \sum_{j=1}^m Q_{x,j}^2$, $\X \gets \emptyset$,
	Gram-Schmidt basis $\bS \gets []$ \;
	Initialise alias table for Walker's alias algorithm to sample from $\pione$.\\
	\ForEach{$t \in 1\ldots m$}
	{
		accept $\gets$ false \;
		\While{not accept}
		{
			Draw $x$ from $\pione$ using the alias method\;
			Compute acceptance ratio $r = 1 - \frac{1}{\pione(x)}\sum_{j=1}^{t-1} 
			(\bQ_{x,:} \bs_t )^2$ \;
			\If{rand() $< r$}{accept $\gets$ true} 
		}
		Add $x$ to $\X$ \;
		Compute residual $\bz_t = (\bI - \bS\bS^\top)(\bQ_{x,:})^\top$ \;
		Add column $\bs_t = \frac{\bz_t}{\norm{\bz_t}}$ to $\bS$ \;
	}
	\label{alg:sampling-proj-rejection}
\end{algorithm}
 Tallying everything we obtain the following theorem. 
\begin{theorem}
	\label{thm:runtime}
	Alg.~\ref{alg:sampling-proj-rejection} samples a projection DPP, with an
	expected runtime scaling as $\O(nm +m^3
	\log m)$. Also, any additional sample from the same DPP can be obtained in an extra $\O(m^3\log m)$ expected runtime. \\
	Moreover, these expected runtimes are representative. Indeed, 
	$\forall\delta\in\left(0,\frac{1}{2}\right)$, the total number of proposals $R=\sum_{t=1}^m R_t$ satisfies, with probability greater than $1-\delta$:
	$$R \leq 2 m \log m + 3m\log\frac{1}{\delta}$$
\end{theorem}
\begin{proof}
	The fact that Alg.~\ref{alg:sampling-proj-rejection} samples a projection DPP in $\O(nm +m^3
	\log m)$ expected runtime is proven above the Theorem's statement. 
	The fact that any additional sample from the same DPP only costs an extra $\O(m^3\log m)$ expected runtime comes from the observation that all initialisation steps (the computation of $\pione$ and the setting-up cost of Walkers' algorithm) have already been  computed for the first sample. To obtain any extra sample, one only needs to run the \emph{for} loop once more, costing $\O(m^3\log m)$. \\
	The final statement relates to concentration properties of $R$.  For $m=1$, it is trivial ($\log$ refers to the natural logarithm in the result) as $R=1$ with probability $1$ in this case. We now show the result for $m\geq 2$. 
	At step $t$ of Alg.~\ref{alg:sampling-proj-rejection} the number of proposals $R_t$ is
	a random geometric variable with parameter $(m-t+1)/m$. Note that all the $R_t$'s are independent. 
	We study here the behavior of $R=\sum_{t=1}^{m} R_t$, the total number of rejection sampling steps in the whole course of Alg.~\ref{alg:sampling-proj-rejection}. 
	The following one-tailed upper bound is drawn from Thm 2.3 in~\cite{Jans18}:
	\begin{equation*}
	\forall\lambda\geq1,\quad p\left( R \geq \lambda E[R]\right) \leq \frac{1}{\lambda} \left( 1-\frac{1}{m} \right)^{(\lambda- 1-\log\lambda)\E[R]}
	\end{equation*}
Let $\delta\in\left(0,\frac{1}{2}\right)$. 
We look for $\lambda$ large enough such that $p\left( R \geq \lambda E[R]\right)\leq \delta$, that is:
\begin{align*}
	-\log\lambda +(\lambda-1-\log\lambda)E[R]\log\left(1-\frac{1}{m}\right)\leq \log\delta
\end{align*}
As $\lambda\geq 1$, $-\log\lambda\leq 0$. It thus suffices to seek $\lambda$ verifying:
\begin{align}
	\label{eq:to_verify}
	(\lambda-1-\log\lambda)E[R]\log\left(1-\frac{1}{m}\right)\leq \log\delta
\end{align}
Note that $E[R]\log\left(1-\frac{1}{m}\right)$ is negative so we seek a lower bound of $\lambda-1-\log\lambda$. One has\footnote{In fact, the function $\lambda-1-\log\lambda$ is convex for all $\lambda\geq 1$ and thus lower-bounded by all its tangents. The one we use is the tangent in $3/2$. Other choices lead to other constants in the result.}: $$\forall\lambda\geq 1, \qquad\frac{1}{3}\lambda-\log \frac{3}{2} \leq \lambda-1-\log\lambda$$
such that Eq.~\eqref{eq:to_verify} is verified provided that:
$$\lambda\geq 3\left(\log \frac{3}{2} -  \frac{\log\delta^{-1}}{\E[R]\log\left(1-\frac{1}{m}\right)}\right).$$
Stated differently, setting $\lambda$ to this lower bound implies $p(R \leq \lambda \E[R])\geq 1-\delta$. 
All is left to show is that $\forall m\geq2$:
\begin{align}
	\label{eq:to_verify2}
3&\left(\log\frac{3}{2} -  \frac{\log\delta^{-1}}{\E[R]\log\left(1-\frac{1}{m}\right)}\right)\E[R] \nonumber\\ &\qquad\qquad\qquad\qquad\leq 
	2 m \log m + 3m\log\delta^{-1}.
\end{align}
 For this, we use two upper bounds. The first one is $$\forall m>1, \qquad-\frac{1}{\log\left(1-\frac{1}{m}\right)}\leq m-1/2.$$  
 The second one is the following bound on $\E[R]$. As $\E[R_t]=m/(m-t+1)$, one has (see, \emph{e.g.}, Chapter 6 of~\cite{abramowitz1964handbook}): $$\frac{\E[R]}{m}= \sum_{t=1}^m \frac{1}{m-t+1}=\sum_{t=1}^{m}\frac{1}{t}=\gamma+\psi(m)+\frac{1}{m}$$
where $\gamma\approx 0.577$ is Euler's constant and $\psi(m)$ the digamma function.  Now, a known bound on $\psi(m)$ is $\psi(m)\leq \log m - \frac{1}{2m}$, which gives $\E[R]\leq m(\log m +\gamma) + \frac{1}{2}$. Applying these two upper bounds yields~Eq.~\eqref{eq:to_verify2}.
\end{proof}

\subsection{Some refinements}
\label{sec:refinements}

Alg.~\ref{alg:sampling-proj-rejection} works well enough as is but there
are some refinements that can further reduce the computational cost.

\subsubsection{Preprocessing for general DPPs}
\label{sec:preprocessing}

In some applications we require several samples from the same DPP, and
algorithms have been described that trade higher set-up cost for a lower cost
per sample (see, \textit{e.g.},~\cite{gillenwater2019tree}).
If the target DPP is a projection DPP, then setting up Alg.~\ref{alg:sampling-proj-rejection} for repeated sampling could not be easier, as stated in Thm~\ref{thm:runtime}: computing $\pione$ and setting up the alias table is part of the preprocessing, so the first sample from the DPP costs $\O(nm+m^3\log m)$ but
after that the cost is just $\O(m^3\log m)$ per sample. 

If the target DPP is not a projection DPP, then one has to use the mixture
representation (Thm~\ref{houghs_thm}): draw a random set of eigenvectors
$\Y$ and run Alg.~\ref{alg:sampling-proj-rejection} with $\bQ = \bU_{\Y,:}$.
Since the kernel changes every time, so does $\pione$, and it cannot be computed
as part of pre-processing. However, the kernels encountered in practice tend to
have rapidly decreasing eigenvalues, so that the variance in $\Y$ is quite small
and the DPP is close to a projection DPP. Without getting into too much detail,
it is possible to pre-compute the partial sum 
\[ \widehat{\pione}(i) = \sum_{j \in \hat{\Y}} U_{i,j}^2 \]
for some highly likely subset of $\Y$, denoted here $\hat{\Y}$. $\pione$ for the
actual sampled $\Y$ can be obtained efficiently by removing the extra entries
and adding the missing ones. The alias table can be computed from scratch.
This type of preprocessing brings down the cost to $\O(ns + m^3\log
m)$ per sample, where $m = \E(|\Y|)$ and $s$ is the expected size of the symmetric
difference between $\hat{\Y}$ and $\Y$. 

\subsubsection{Caching computations and updating the proposal distribution}
\label{sec:updating-prop}

Clearly, Alg.~\ref{alg:sampling-proj-rejection} has some wasted
computation, since all the computations done when a proposal $x$ is rejected are
performed again should $x$ come up a second time. When $n$ is small, or when
$\pione$ has low entropy, this may indeed happen several times. 
Caching is one way of reducing the amount of redundant computations that are
performed. Going back to the recursive formula for $\pit$ (Eq.
(\ref{eq:unnormalised-recursion})), we see that it is cheaper to compute $\pit$
from $\pi^{(t-1)}$ than it is to compute it from scratch. A reasonable caching
strategy is then to keep track for every point $x$ of the last density
evaluation performed for that point. If $x$ comes up again, the evaluation of
the acceptance ratio is simplified.

Another natural idea is to update the proposal distribution over the course of
the algorithm (instead of sticking with $\pione$ throughout). The most basic
version is that any $x$ that has already been selected has an acceptance
probability of 0, so we may as well not suggest them. Another is that points
similar to a selected point are quite unlikely to come up further down, and
so it may be worth computing $\pit$ for these neighbours to tighten the bound.
Finally, we may combine this idea with the caching idea, which provides a better
bound for every point that has ever been suggested. Unfortunately this runs
against the difficulty of updating the alias table in Walker's algorithm, which
one would have to compute from scratch at every update (at cost 
$\O(n)$). A better way would be to use a binary tree representation
\cite{devroye86nonuniformrandom}, which can be updated at cost $\O(\log n)$ and
provides samples also at cost $\O(\log n)$. The implementation complexity
increases a lot however, and we have not pursued this further. As we shall see
below, Alg.~\ref{alg:sampling-proj-rejection} is quite fast in practice,
and implementation time may be better invested in feature computation and
orthogonalisation.

\section{Empirical results}
\label{sec:empirical-results}

\begin{figure*}
	\centering
	\includegraphics[width=\textwidth]{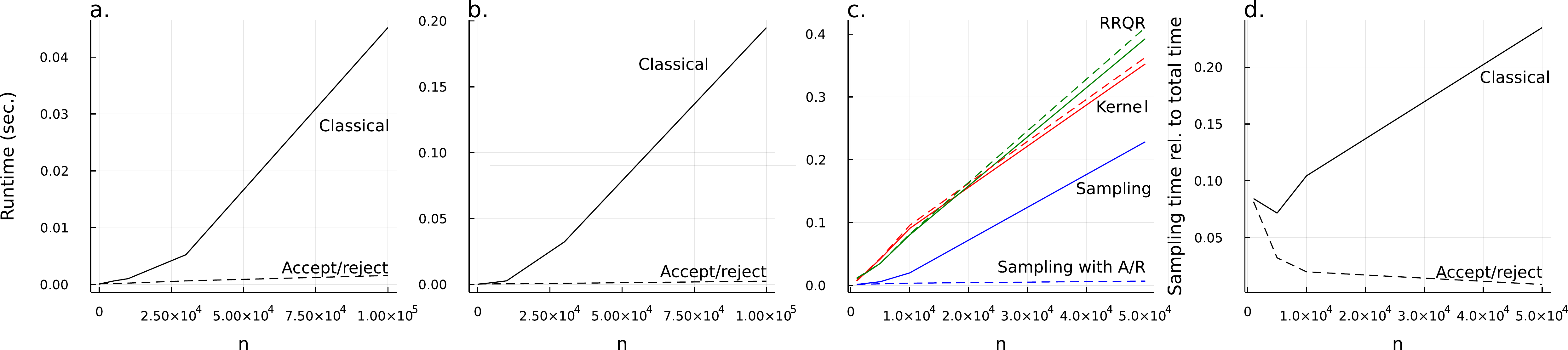}
	\caption{\emph{Left panels}: median time needed to sample a projection DPP,
		using the standard approach (Alg.~\ref{alg:sampling-proj}) vs. A/R (Alg.~\ref{alg:sampling-proj-rejection}). For two values of $m$: a) $m=30$, b) $m=60$. Note
		that here the time taken to compute an orthogonal basis is not taken into
		account (see text). \emph{Right panels}: simulation of a full workload that includes feature
		generation and orthogonalisation (see text). c) time taken by each step in
		the computation as a function of $n$. Solid lines are runtimes when sampling
		using the classical method, dashed, when using A/R. Note that the first two
		steps (labeled ``kernel'' and ``RRQR'', in red and green respectively) are identical. A/R sampling becomes
		beneficial at around $n=1,000$ and is orders of magnitude faster at
		$n=100,000$. d) Sampling time as a function of total computation time. }
	\label{fig:results_runtime}
\end{figure*}

We compare the Accept/Reject algorithm (Alg.~\ref{alg:sampling-proj-rejection}) to its classical counterpart (Alg.~\ref{alg:sampling-proj}) for
different values of $n$ and $m$. Both algorithms are implemented in the Julia
language and are publicly available\footnote{we’ve added a folder in our DPP.jl
  repository containing the code necessary to reproduce the figures, available
  here: \url{https://github.com/dahtah/DPP.jl/tree/main/misc/sampling_paper}. In
addition, the A/R sampler is available as part of the DPP.jl software
package, \url{https://github.com/dahtah/DPP.jl/}.}. For each value of $n$, we
sample a random projection matrix of size $n \times m$ (via QR decomposition of
a matrix with Gaussian entries). We then pre-compute the leverage scores, and
run each algorithm 100 times. Fig. \ref{fig:results_runtime} a) and b) show the measured 
median runtimes. All tests are run on a 2017 Linux
laptop with i7-8550U Intel CPU and 8 Go of RAM.

For very small values of $n$, the classical algorithm is faster, which can be
explained by the efficiency of BLAS calls. At each step the whole conditional
distribution is computed, the main cost being a matrix multiplication (i.e., a
BLAS call), which benefits from efficient multithreaded code. However, the
different asymptotic scalings ($\O(nm^2)$ vs. $\O(nm)$) soon makes the classical
algorithm uncompetitive. The cross-over point in our simulations is at around
$n=1,000$, and by $n=10,000$ the difference is stark. This illustrates the ``very significant speed-up'' 
scenario described in the introduction: $p=m$ and orthogonalisation is already computed. We now move on to illustrate a less favorable case related to the ``moderate speed-up'' 
 scenario of the introduction: $p$ equals a few times $m$. 

In many cases, sampling a projection DPP is only one of the steps in a process
that involves feature computation and orthogonalisation. For instance, in
\cite{Tremblay:DPPforCoresets}, a DPP based on the Gaussian kernel is used to
produce a subset of the data suitable for running clustering algorithms.
Starting from $n$ points, $\bx_1,\ldots,\bx_n$ in $\R^d$ they use a DPP with
L-ensemble given by $L_{ij} = \gamma \kappa(\vect{x},\vect{y}) = \gamma \exp \left( - \frac{1}{\sigma^2} \norm{\vect{x}-\vect{y}}^2 \right)$, 
where $\gamma$ is a tuning parameter that determines the expected size.
Producing a sample from this exact DPP requires the eigendecomposition of $\bL$
which is impractical; however, $\bL$ is numerically low-rank for relevant values
of $\sigma$ and this can be exploited. In \cite{Tremblay:DPPforCoresets} a low-rank approximation of $\bL$ is used, based on random Fourier features
\cite{rahimi_random_2008} followed by a SVD, bringing down the total cost to
$\O(nm^2)$. We now sketch (without any formal justification) another procedure
which gives comparable results at lower cost.

Gaussian kernel matrices have rapidly decaying spectra (see, \textit{e.g.},
\cite{wathen2015eigenvalues}), which implies in particular that DPPs
sampled from a Gaussian L-ensemble are well approximated by projection DPPs with
kernels $\bP_m= \bU_m\bU_m^\top$ where $\bP_m$ projects onto the dominant
eigenspace of $\bL$ of order $m$. Thus, all we need is a good basis for
the dominant eigenspace. Methods from randomised linear algebra offer good
practical tools (``range finders'') to obtain a basis for such a space~\cite{martinsson2020randomized}. For these simulations, we used the following
approximation:
\begin{enumerate}
	\item Select (and compute)  $5m$ columns uniformly from $\bL$. Call this matrix $\bA$.
	We call this the ``kernel step''. 
	\item Use Rank-Revealing QR (RRQR, \cite{chan1987rank}) and random
	projections, as implemented in the Julia package LowRankApprox.jl, to produce
	$\bQ$, an orthonormal matrix of size $n \times m$ that approximates the image
	of $\bA$. We call this the ``RRQR step''.
	\item Sample a DPP with projection kernel $\bQ\bQ^\top$ using either the
	classical or the A/R algorithm. 
\end{enumerate}
We set $m=100$ and time each step. This results in a total
runtime of around 1.2 sec. at $n=10^5$ with the A/R sampler, which challenges
the notion that DPPs are very slow to sample from. With this procedure, the time
spent sampling the actual DPP goes up to ~20\% of total time for the classical
algorithm at $n=10^5$, but using the A/R sampler sampling time becomes
negligible. See Fig.~\ref{fig:results_runtime} c) and d) to see how these times vary with $n$. 
This indicates that for some computations the implementation effort
may be better allocated to speeding up the linear algebra and
feature computation part rather than the sampling part. In this particular
instance, step (1) at least could be sped up by exploiting
parallelism, or the GPU, which we did not attempt.

\section{Discussion and perspectives}
\label{sec:m-vs-mlogm}

On top of the improvement on the sampling time of DPPs, our results imply the following intriguing by-product. Let $\X$ be a projection DPP of size $m$. A set of $\O(m \log m)$ points sampled i.i.d. from the inclusion probability distribution $p(i \in \X)=\pione(i)/m$ (also known as \emph{leverage scores}) contains with high probability a realisation from the DPP. 
 The consequences of this fact are worth discussing.
First, let us put the result a bit more formally. 

\begin{definition}
	Let $\X$ be a DPP on $\Omega$ and $\Y\subseteq\Omega$. We call $\Phi(\Y)$ a \emph{thinning algorithm} if it returns a subset of $\Y$. Moreover, 
	we say $\Phi(\Y)$ is \emph{successful} when it  returns a realisation from $\X$. 
\end{definition}

\begin{corollary}
	\label{cor:guaranteed}
	Let $\X$ be a projection DPP of size $m$, and $\Y$ be a set of i.i.d. points sampled with replacement with probability proportional to the leverage scores: $p(i \in \X)=\pione(i)/m$. Let $\delta\in(0,1/2)$. A simple modification of Alg.~\ref{alg:sampling-proj-rejection} gives a thinning algorithm $\Phi$ that verifies: $\Phi(\Y)$ is successful with probability greater than $1-\delta$ provided that $|\Y|\geq 2 m \log m + 3m\log\left(\frac{1}{\delta}\right)$.
\end{corollary}

\begin{proof}
	Let $\Y$ be drawn i.i.d. with replacement from the leverage score distribution $\pione$. 
	$\Phi$ is the following simple modification of Alg.~\ref{alg:sampling-proj-rejection}. Instead of drawing a new proposal $x$ using the alias method at the beginning of the \emph{while} loop as in Alg.~\ref{alg:sampling-proj-rejection}, draw uniformly and without replacement from $\Y$. If $\Phi$ finishes before emptying $\Y$, then it is successful. If $\Y$ is empty and $\Phi$ is not terminated, then it fails. 
	The probability that $\Phi$ succeeds is thus equal to the probability that $|\Y|$
    is larger than the number of proposals $R$ of Alg.~\eqref{alg:sampling-proj-rejection}. 
    By Thm~\ref{thm:runtime}, setting $|\Y|= 2 m \log m + 3m\log\left(\frac{1}{\delta}\right)$ yields $p(|\Y|\geq R)\leq 1-\delta$ and ends the proof.
\end{proof}
%
Note that this is a substantial improvement over the work of \cite{derezinski2019exact}, which gives this result only for  $|\Y| \geq \O(m^2)$ i.i.d. points. 

A natural question is to ask if this result is optimal: can we find a thinning algorithm that succeeds with high probability for even smaller i.i.d sets? The answer is no  in general: 
\begin{proposition}
	Corollary~\ref{cor:guaranteed} is optimal in the following sense. 
	Let $\X$ and $\Y$ be as previously. There does not exist a generic thinning algorithm able to succeed with fixed non-null probability if  $|\Y| = o(m \log m)$.
\end{proposition}

\begin{proof}
	We show the proposition by exhibiting a type of DPP for which there does not exist a thinnning algorithm that succeeds with a non-null probability if $|\Y|$ is asymptotically smaller than $m\log m$.
	
	It is well-known in the folklore that a form of \emph{stratified sampling} is a special case of projection DPPs. In stratified sampling, we partition the ground set $\Omega$ into $m$ classes, and sample an item uniformly from each segment of the partition. To simplify the argument, assume $\Omega$ can be cut into $m$ subsets of equal size, and define vector $\vect{e}_j$ as the (normalised) indicator of segment $j$, i.e. $\vect{e}_j(i) = \sqrt{\frac{n}{m}}$  if item $i$ is in segment $j$ and $0$ otherwise. Let $\matr{E} = \left[ \vect{e}_1 \ldots \vect{e}_m \right]$. Then it is easy to show that stratified sampling is equivalent to a DPP $\X$ with marginal kernel $\bK = \matr{E}\matr{E}^\top$. Since $\matr{E}^\top\matr{E} = \bI$, the DPP in question is a projection DPP. 
	
	Because of the nature of stratified sampling, we know that $\X$ contains a point from each one of the segments, and that $p(i \in \X) = \frac{m}{n}$ for all $i$. Now in order for any thinning algorithm to produce a stratified sample, the i.i.d. sample $\Y$ needs to contain at least one point from each segment. Let $l \eqdef |\Y|$: how large does $l$ need to be so that $\Y$ contains at least one point from each segment with probability at least $\alpha$ ($\alpha>0$ fixed)? This is an instance of the coupon collector's problem. Assume that at each time $t$ we add a ball to one of $m$ urns with equal probability, and call $T$ the smallest $t$ such all urns have at least one ball. We show that for any $l(m) = o(m\log m)$, $\lim_{m \to \infty} p(T < l(m) ) = 0$. 
	
	To do this, we need an upper bound for $p(T < l(m))$. One could work with results from~\cite{witt2014fitness} for instance. However, we prefer an elegant line of proof inspired by a contribution of a stackexchange user called ``cardinal''\footnote{see \url{https://stats.stackexchange.com/q/7774}}. $T$ can be viewed as a sum of $m$ geometric variables: $T= \sum_{i=1}^m T_i$, where $T_i$ is the time at which $i$ urns have at least one ball. 
	All the $T_i$'s are independent geometric random variables with success probability $p_i = 1-\frac{i-1}{m}$. Indeed, the same representation is obtained by considering Alg.~\ref{alg:sampling-proj-rejection} in the special case of stratified sampling (each iteration fills one urn).
	Now, Markov's inequality gives:
	\begin{align*}
	\forall s>0,\qquad	p(T < l) &= p(\exp(-sT) > \exp(-sl)) \\
		&\leq \exp(sl) \E \left(\exp(-sT)\right)
	\end{align*}
	Since  $T$ is a sum of independent geometric variables, $\E(\exp(-sT))$ is easy to compute\footnote{Using $E\left(\exp^{-sT}\right) = \prod_{j=1}^m E\left(\exp^{-sT_j}\right) = \prod_{j=1}^m \frac{p_j\exp^{-s}}{1-(1-p_j)\exp^{-s}}$ and changing variable $i \leftarrow m-j+1$}:
	\[ \E(\exp(-sT)) = \prod_{i=1}^m \frac{i}{m(e^s - 1) + i}\]
	Picking $s = \frac{1}{m}$, we obtain:
	\[ p(T < l)  \leq \exp^\frac{l}{m} \prod_{i=1}^m \frac{i}{m(e^{1/m} - 1) + i} \]
	Since $e^{\frac{1}{m}} \geq 1+\frac{1}{m}$, we upper bound the right-hand side to:
	$$p(T < l)  \leq \exp^\frac{l}{m} \prod_{i=1}^m \frac{i}{1 + i} = \frac{\exp(\frac{l}{m})}{m+1}$$
	Thus, any choice of sample size $l(m)$ such that $\frac{\exp(\frac{l(m)}{m})}{m+1}$ goes to 0 in the limit is asymptotically too small (the probability of success goes to 0). Noting that $\frac{\exp(\frac{l}{m})}{m+1} = o(1)$ is equivalent to $l(m) = o(m \log m)$ yields the claim.
\end{proof}

This transition occuring at $\O(m\log m)$ calls for discussion, and paves the way to future interesting lines of research.
First of all, Corollary~\ref{cor:guaranteed} shows, from an original angle, that
the repulsiveness of DPPs is weak. Indeed, other repulsive processes such as hard-core processes cannot verify such property in all generality. For instance, in the high density limit of a hard-sphere model, the probability that the position of $m$ non-overlapping spheres can be found within a set of only $\O(m \log m)$ iid points drawn uniformly, tends to $0$. 
In addition, these results ask the following question: in what cases should one pay the extra cost of sampling $m$ elements from a DPP, rather than simply sampling $\O(m\log m)$ elements i.i.d. from the leverage score distribution? Of course, when the objective is to sample a diverse set, such as in search engines, it is always worthwhile to sample the DPP. However, in the case of integration~\cite{bardenet2016monte, coeurjolly2020monte}; or in the case of coresets~\cite{Tremblay:DPPforCoresets}, the answer is not so clear and requires further investigation.

\subsubsection*{Acknowledgements}
We thank the five anonymous reviewers for their helpful comments that led to an improved version of this manuscript. 
This work was partially supported by the ANR project GRANOLA (ANR-21-CE48-0009), as well as the LabEx PERSYVAL-Lab (ANR-11-LABX-0025-01) and MIAI@Grenoble Alpes (ANR-19-P3IA-0003).

\bibliography{biblio}
\end{document}